\relax
\documentclass[letterpaper]{article} 
\usepackage{aaai20}  
\usepackage{times}  
\usepackage{helvet} 
\usepackage{courier}  
\usepackage[hyphens]{url}  
\usepackage{graphicx} 
\usepackage{graphicx}  
\usepackage{amsmath}
\usepackage{amssymb}
\usepackage{enumitem}
\usepackage{amsthm}
\usepackage{algorithm}
\usepackage{algorithmic}
\usepackage{mathtools}

\DeclareMathOperator*{\argmax}{arg\,max}
\newtheorem{theorem}{Theorem}[]

\DeclarePairedDelimiterX{\infdivx}[2]{(}{)}{%
  #1\;\delimsize\|\;#2%
}

\allowdisplaybreaks

\title{Sequential Mode Estimation with Oracle Queries}
\author{Dhruti Shah,\textsuperscript{\rm 1} Tuhinangshu Choudhury,\textsuperscript{\rm 1} Nikhil Karamchandani,\textsuperscript{\rm 1} Aditya Gopalan\textsuperscript{\rm 2}\\
\textsuperscript{\rm 1}Indian Institute of Technology, Bombay\\
\textsuperscript{\rm 2}Indian Institute of Science, Bangalore\\
dhruti96shah@gmail.com, choudhurytuhinangshu@gmail.com \\
nikhilk@ee.iitb.ac.in, aditya@iisc.ac.in 
}
 
\begin{document}

\maketitle

\begin{abstract}
We consider the problem of adaptively PAC-learning a probability distribution $\mathcal{P}$'s mode by querying an oracle for information about a sequence of i.i.d. samples $X_1, X_2, \ldots$ generated from $\mathcal{P}$. We consider two different query models: (a) each query is an index $i$ for which the oracle reveals the value of the sample $X_i$, (b) each query is comprised of two indices $i$ and $j$ for which the oracle reveals if the samples $X_i$ and $X_j$ are the same or not. For these query models, we give sequential mode-estimation algorithms
which, at each time $t$, either make a query to the corresponding oracle
based on past observations, or decide to stop and output
an estimate for the distribution's mode, required to be correct with a specified confidence. We analyze the query complexity of these algorithms for any underlying distribution $\mathcal{P}$, and derive corresponding lower bounds on the optimal query complexity under the two querying models.
\end{abstract}

\section{Introduction}
\noindent Estimating the most likely outcome of a probability distribution is a useful primitive in many computing applications such as counting, natural language processing, clustering, etc. We study the probably approximately correct (PAC) sequential version of this problem in which the learner faces a stream of elements sampled independently and identically distributed (i.i.d.) from an unknown probability distribution, wishing to learn an element with the highest probability mass on it (a mode) with confidence. At any time, the learner can issue queries to obtain information about the identities of samples in the stream, and aims to use as few queries as possible to learn the distribution's mode with high confidence. 
Specifically, we consider two natural models for sample identity queries -- (a) each query, for a single sample of the stream so far, unambiguously reveals the identity ({\em label}) of the sample, (b) each query, for a pair of samples in the stream, reveals whether they are the same element or not.

A concrete application of mode estimation (and one of the main reasons that led to this formulation) is the problem of adaptive, {\em partial} clustering, where the objective is to find the largest cluster (i.e., equivalence class) of elements as opposed to learning the entire cluster grouping \cite{mazumdar2017clustering,mazumdar2017query,mazumdar2017theoretical,mazumdar2016clustering}. We are given a set of elements with an unknown clustering or partition, and would like to find the elements comprising the largest cluster or partition. Suppose a stream of elements is sampled uniformly and independently from the set, and at each time one can ask a {\em comparison oracle} questions of the form: ``Do two sampled elements $u$ and $v$ belong to the same cluster or not?" Under this uniformly sampled distribution for the element stream, the probability of an element belonging to a certain cluster is simply proportional to the cluster's size, so learning the heaviest cluster is akin to identifying the mode of the distribution of a sampled element's cluster label. 

We make the following contributions towards understanding the sequential query complexity for estimating the mode of a distribution using a stream of samples. (a) For both the individual-label and pairwise similarity query models, we give sequential PAC query algorithms which provably output a mode of the sample-generating distribution with large probability, together with guarantees on the number of queries they issue. These query complexity upper bounds explicitly depend on parameters of the unknown discrete probability distribution, in that they scale inversely with the gap between the probability masses at the mode and at the other elements in the distribution's support. The proposed algorithms exploit the probabilistic i.i.d. structure of the data stream to resolve uncertainty about the mode in a query-efficient fashion, and are based on the upper and lower confidence bounds (UCB, LCB) principle from online learning to guide adaptive exploration across time; in fact, we employ more refined empirical Bernstein bounds \cite{maurer2009empirical} to take better advantage of the exact structure of the unknown sample distribution.   (b) We derive fundamental limits on the query complexity of any sequential mode-finding algorithm for both query models, whose constituents resemble those of the query complexity upper bounds for our specific query algorithms above. This indicates that the algorithms proposed make good use of their queries and the associated information in converging upon a mode estimate. (c) We report numerical simulation results that support our theoretical query complexity performance bounds. 

\subsection{Related Work}
The mode estimation problem has been studied classically in the batch or non-sequential setting since many decades back, dating to the work of Parzen \cite{parzen62estimation} and Chernoff \cite{chernoff1964estimation}, among others. This line of work, however, focuses on the objective of consistent mode estimation (and the asymptotic distribution of the estimate) for continuous distributions, instead of finite-time PAC guarantees for large-support discrete distributions as considered here. Our problem is essentially a version of sequential composite hypothesis testing with adaptive ``actions" or queries, and with an explicit high-confidence requirement on the testing algorithm upon stopping. 

There has been a significant amount of work in the streaming algorithms community, within computer science, on the "heavy hitter" problem -- detecting the most frequent symbol in an arbitrary (non-stochastic) sequence -- and generalizations thereof pertaining to estimation of the empirical moments, see e.g., \cite{misra1982finding}, \cite{karp2003simple}, \cite{manku2002approximate}. However the focus here is on understanding resource limits, such as memory and computational effort, on computing on arbitrary (non stochastic / unstructured) streams that arise in highly dynamic network applications. We are instead interested in quantifying the {\em statistical} efficiency of mode estimation algorithms in terms of the structure of the generating probability distribution. 

Adaptive decision making and resolution of the explore-exploit trade off is the subject of work on the well-known multi-armed bandit model, e.g., \cite{bubeck2012regret}. At an abstract level, our problem of PAC-mode estimation is like a multi-armed bandit ``best arm" identification problem \cite{kaufmann2016complexity} but with a different information structure -- queries are not directly related to any utility structure for rewards as in bandits. 

Perhaps the closest work in spirit to ours is the recent work by Mazumdar and co-authors  \cite{mazumdar2017clustering,mazumdar2017query,mazumdar2017theoretical,mazumdar2016clustering}, where the aim is to learn the entire structure of an unknown clustering by making information queries. In this regard, studying the mode estimation problem helps to shed light on the simpler, and often more natural, objective of merely identifying the largest cluster in many machine learning applications, which has not been addressed by previous work.  

\section{Problem formulation} \label{problem_formulation}
In this section we develop the required notation and describe the query models. \\
Consider an underlying unknown discrete probability distribution $\mathcal{P}$ with the support set $\{ 1,2,...k \}$. For each $i \in \{1,2,\ldots,k\}$ and a random variable $X \sim \mathcal{P}$, let $Pr(X=i)= p_i(\mathcal{P}) \equiv p_i$. 

We would like to estimate the {\em mode} of the unknown distribution $\mathcal{P}$, defined as any member of the set\footnote{$\arg\max_{i \in S} p_i$ is used to denote the set of all maximisers of the function $i \to p_i$ on $S$.} $\arg\max_{1 \leq i \leq k}\, p_i$. Towards this, we assume query access to an {\em oracle} containing a sequence of independently and identically distributed (i.i.d.) samples from $\mathcal{P}$, denoted $X_1, X_2, \ldots$ We study the mode estimation problem under the following query models to access the values of these i.i.d. samples:
\begin{enumerate}
    \item \textbf{Query Model 1 (QM1) :} For each query, we specify an index $i \ge 1$ following which the oracle reveals the value of the sample $X_i$ to us. Since the samples are i.i.d., without loss of generality, we will assume that the $t^{th}$ successive query reveals $X_t$, $t = 1, 2, \ldots$
    \item \textbf{Query Model 2 (QM2) :} In this model, the oracle answers {\em pairwise similarity} queries. For each query, we specify {\em two} indices $i, j \in \{1, 2, \ldots \}$, following which the oracle reveals if the two samples $X_i$ and $X_j$ are equal or not. Formally, the response of the oracle to a query $(i, j)$ is
    \[ 
\mathcal{O}(i,j) = \left\{ \begin{array}{ll}
         +1 & \mbox{if $X_i=X_j$},\\
        -1 & \mbox{otherwise}.\end{array} \right.
    \]

    Note that to know the value of a sample $X_i$ in this query model, multiple pair-wise queries to the oracle might be required.
\end{enumerate}
For each of the query models above, our goal is to design a statistically efficient sequential mode-estimation algorithm which, at each time $t$, either makes a query to the oracle based on past observations or decides to stop and output an estimate for the distribution's mode. Mathematically, a sequential algorithm with a stopping rule decides an action $A_t$ at each time $t \ge 1$ depending only on past observations. For QM1, $A_t$ can be one of the following:
\begin{itemize}
\item \textit{(continue,$t$)}: Query the index $t$,
\item \textit{(stop,$\hat{m}$)}, $\hat{m} \in \{1, \ldots, k\}$: Stop querying and return $\hat{m}$ as the mode estimate.
\end{itemize}
For QM2, $A_t$ can be one of the following:
\begin{itemize}
\item \textit{(continue,$t$)}: Continue with the next round, with possibly multiple sequential pairwise queries of the form $(t, j)$ for some $j < t$. That is, we compare the sample $X_t$ with some or all of the previous samples.
 
\item \textit{(stop,$\hat{m}$)}, $\hat{m} \in \{1, \ldots, k\}$: Stop querying and return $\hat{m}$ as the mode estimate.
\end{itemize}
The stopping time of the algorithm is defined as 
\begin{align*}
\tau := \inf\{ t \geq 1 : A_t = (stop,.) \}.
\end{align*} 
The cost of the algorithm is measured by its {\em query complexity} -- the number of queries made by it before stopping.
For $\delta > 0$, a sequential mode-estimation algorithm is defined to be a \textbf{$\delta$-true mode estimator} if it correctly identifies the mode for every  distribution $\mathcal{P}$ on the support set $\{1,2,\ldots,k\}$ with probability at least $1 - \delta$, i.e., $\mathbb{P}_{\mathcal{P}}[\hat{m} \in \arg\max_{1 \leq i \leq k} p_i(\mathcal{P})] \geq 1-\delta$. The goal is to obtain $\delta$-true mode estimators for each query model (QM1 and QM2) that require as few queries as possible. For a $\delta$-true mode estimator $\mathcal{A}$ and a distribution $\mathcal{P}$, let $Q_{\delta}^{\mathcal{P}}(\mathcal{A})$ denote the number of queries made by a $\delta$-true mode estimator when the underlying unknown distribution is $\mathcal{P}$. We are interested in studying the optimal query complexity of $\delta$-true mode estimators. Note that $Q_{\delta}^{\mathcal{P}}(\mathcal{A})$ is itself a random quantity, and our results either hold in expectation or with high probability. \\
For the purpose of this paper, we assume that $p_1>p_2\geq....\geq p_k$ i.e. the mode of the underlying distribution is $1$, and hence a $\delta$-true mode estimator returns $1$ with probability at least $(1-\delta)$. In Sections~\ref{sec_QM1} and \ref{sec_QM2}, we discuss $\delta$-true mode estimators and analyze their query complexity for the QM1 and QM2 query models respectively.  We provide some experimental results in Section~\ref{sec:exp} and further explore a few variations of the problem in Section~\ref{discussion}. Several proofs have been relegated to the Appendix.

\section{Mode estimation with QM1} \label{sec_QM1}
We will begin by presenting an algorithm for mode estimation under QM1 and analyzing its query complexity.
\subsection{Algorithm} \label{algo_qm1}
Recall that under the QM1 query model, querying the index $t$ to the oracle reveals the value of the corresponding sample, $X_t$, generated i.i.d. according to the underlying unknown distribution $\mathcal{P}$. During the course of the algorithm, we  form bins for each element $i$ in the support $\{1,2,\ldots,k\}$ of the underlying distribution. Bin $j$ is created when the first sample with value $j$ is revealed and any further samples with that value are `placed' in the same bin. For each query $t$ and revealed sample value $X_t$, define $Z^i_t$ for $i \in \{1,2,...k\}$ as follows.
\begin{equation*}
 Z^i_t =
  \begin{cases}
    1  & \quad \text{if }X_t=i\\
    0  & \quad \text{otherwise.}
  \end{cases} 
\end{equation*}
Note that for each given $i,t$, $Z^i_t$ is a Bernoulli random variable with $E[Z_t^i] = p_i$. Also for any given $i$, $\{Z_t^i\}$ are i.i.d. over time.\\
Our mode estimation scheme is presented in Algorithm~\ref{alg1}. At each stage of the algorithm, we maintain an empirical estimate of the probability of bin $i$, $p_i$, for each $i \in \{1,2,\ldots,k\}$. Let $\hat{p}_i^t$ denote the estimate at time $t$, given by  
\begin{equation}\label{p_hat}
\hat{p}_i^t=\frac{\sum_{j=1}^t Z^i_j}{t},
\end{equation}
where recall that $Z^i_j$ for $j=1,2....t$ are the $t$ i.i.d. samples. Also, at each time instant, we maintain confidence bounds for the estimate of $p_i$. The confidence interval for the $i^{th}$ bin probability at the $t^{th}$ iteration is denoted by  $\beta_i^t$, and it captures the deviation of the empirical value $\hat{p}_{i}^t$ from its true value $p_i$. In particular, the confidence interval value $\beta_i^t$ is chosen so that the true value $p_i$ lies in the interval $[\hat{p}_i^t - \beta_i^t, \hat{p}_i^t + \beta_i^t]$ with a significantly large probability. The lower and upper boundaries of this interval are referred to as the lower confidence bound (LCB) and the upper confidence bound (UCB) respectively. The particular choice for the value of $\beta_i^t$ that we use for our algorithm is presented in Section~\ref{analysis_qm1}. 

Finally, our stopping rule is as follows : $A_t=\textit{(stop, i)}$ when there exists a bin $i \in \{1,2,\ldots,k\}$ whose LCB is greater than the UCB of all the other bins, upon which the index $i$ is output as the mode estimate. 

Given the way our confidence intervals are defined, this ensures that the output of the estimator is the mode of the underlying distribution with large probability. 

\begin{algorithm} 
\caption{Mode estimation algorithm under QM1} 
\label{alg1} 
\begin{algorithmic}[1] 
    \STATE $t=1$
    \STATE $A_0=$ \textit{(continue,1)} : Obtain $X_1$.
    \LOOP
        \IF{a bin with value $X_t$ already exists}
            \STATE Add query index $t$ to the corresponding bin. 
        \ELSE
            \STATE Create a new bin with value $X_t$.
        \ENDIF
        \STATE Update the empirical estimate $\hat{p}_i^t$ \eqref{p_hat} and confidence interval $\beta_i^t$ \eqref{beta} for all bins $i \in \{1,2,\ldots,k\}$.
        \STATE \begin{equation*}
                 A_t =
                  \begin{cases}
                    \text{\textit{(stop,i)}:} & \hspace{0.2in} \text{if }\exists i, \text{ s.t. } \forall j \neq i\\
                    \vspace{0.15in}\text{Exit, Output $i$}&\hspace{0.25in}\hat{p}_i^t - \beta_i^t > \hat{p}_j^t + \beta_j(t), \\
                    \text{\textit{(continue,t+1)}:} & \hspace{0.2in} \text{otherwise } \\
                    \text{Obtain $X_{t+1}$}
                  \end{cases} 
                \end{equation*}

        \STATE $t=t+1$
    \ENDLOOP    
\end{algorithmic}
\end{algorithm}

\subsection{Analysis} \label{analysis_qm1}
\begin{theorem}\label{analysis_lem}
For the following choice of $\beta_{i}^t$, Algorithm~\ref{alg1} is a $\delta$-true mode estimator. 
\begin{equation}\label{beta}
\beta_{i}^t = \sqrt{\frac{2V_t(Z^{i})\log (4kt^2/\delta)}{t}} + \frac{7 \log (4kt^2/\delta)}{3(t-1)}, 
\end{equation}
where $\displaystyle V_t(Z^{i}) = \frac{1}{t(t-1)} \sum_{1\leq p<q\leq t} (Z^{i}_p - Z^{i}_q)^2 $ is the empirical variance.
\end{theorem}
\begin{proof}
This proof is based on confidence bound arguments. To construct the confidence intervals for the  probability values $\{p_i\}$'s, we use the empirical version of the Bernstein bound given in \cite{maurer2009empirical}. The result used is stated as Theorem~\ref{bern_thm} in Appendix~\ref{emp_bern}. Using the result in our context, we get the following for any given pair $(i,t)$ with probability at least $(1-\delta_1)$: 
\begin{equation}
|p_i - \hat{p}_i^t| \leq \sqrt{\frac{2V_t(Z^i)\log (2/\delta_1)}{t}} + \frac{7 \log (2/\delta_1)}{3(t-1)},
\label{Eqn:confinterv}
\end{equation}
where $V_t(Z^i)$ is the sample variance, i.e., $\displaystyle V_t(Z^i) = \frac{1}{t(t-1)} \sum_{1\leq p<q\leq t} (Z^i_p - Z^i_q)^2 $.\\
To establish confidence bounds on the sample variance, we use the result given in \cite{maurer2009empirical}, which is stated as Theorem~\ref{var} in Appendix~\ref{emp_bern}. Using the result in our context, and noting that the expected value of $V_t(Z^i)$ would be $p_i(1-p_i)$, we get the following for any given pair $(i,t)$ with probability at least $(1-\delta_2)$: 
\begin{equation}
|\sqrt{p_i(1-p_i)} - \sqrt{V_t(Z^i)}| \leq \sqrt{\frac{2\log (1/\delta_2)}{t-1}}.\label{Eqn:confinterv_var}
\end{equation}
Let $\mathcal{E}_1$ denote the error event that for some pair $(i,t)$ the confidence bound \eqref{Eqn:confinterv} around $\hat{p}_i^t$ is violated. Also, let $\mathcal{E}_2$ denote the error event that for some pair $(i,t)$ the confidence bound \eqref{Eqn:confinterv_var} around the sample variance $V_t(Z^i)$ is violated. Choosing $\delta_1 = \delta_2 = \frac{\delta}{2kt^2}$, and taking the union bound over all $i,t$, from \eqref{Eqn:confinterv} and \eqref{Eqn:confinterv_var}, we get $\mathbb{P}[\mathcal{E}_1] \leq \delta/2$ and $\mathbb{P}[\mathcal{E}_2] \leq \delta/2$. Hence we get that
\begin{equation}\label{union}
\mathbb{P}[\mathcal{E}_1^c \cap \mathcal{E}_2^c] \geq 1 - \delta.
\end{equation} 
This means that with probability at least $1-\delta$ the confidence bounds corresponding to both equations \eqref{Eqn:confinterv} and \eqref{Eqn:confinterv_var} hold true for all pairs $(i,t)$.\\
We now show that if the event $\mathcal{E}_1^c \cap \mathcal{E}_2^c$ is true, then Algorithm~\ref{alg1} returns $1$ as the mode. To see this, assume the contrary that the algorithm returns $i \neq 1$. Under $\mathcal{E}_1^c \cap \mathcal{E}_2^c$ the confidence intervals hold true for all pairs $(i,t)$ and hence the stopping condition defined in Line~10, Algorithm~\ref{alg1} will imply
\begin{equation*}
p_i \geq \hat{p}_i^t - \beta_{i}^t > \hat{p}_1^t + \beta_{1}^t \geq p_1 
\end{equation*}
which is false. Thus Algorithm~\ref{alg1} returns $1$ as the mode if the event $\mathcal{E}_1^c \cap \mathcal{E}_2^c$ is true.
Hence, the probability of returning $1$ as the mode is at least $\mathbb{P}[\mathcal{E}_1^c \cap \mathcal{E}_2^c]$, which by \eqref{union} implies that Algorithm~\ref{alg1} is a $\delta$-true mode estimator. 
\end{proof}

\subsection{Query Complexity Upper bound} \label{qm1_ub}
\begin{theorem} \label{lem_qm1_ub}
For a $\delta$-true mode estimator $\mathcal{A}_1$, corresponding to Algorithm~\ref{alg1}, we have the following with probability at least $(1-\delta)$.
\begin{equation*}
{Q}_{\delta}^{\mathcal{P}}(\mathcal{A}_1) \leq \frac{592}{3} \frac{p_1}{(p_1-p_2)^2} \log \left( \frac{592}{3} \sqrt{\frac{k}{\delta}} \frac{p_1}{(p_1-p_2)^2} \right). 
\end{equation*}
\end{theorem}
\begin{proof}
If the event $\mathcal{E}_1^c \cap \mathcal{E}_2^c$ is true, it implies that all confidence intervals hold true. We find a value $T^*$, which ensures that by $t=T^*$, the confidence intervals of the bins have separated enough such that the algorithm stops. Since at each time one query is made, this value of $T^*$ would also be an upper bound on the query complexity $Q_\delta^{\mathcal{P}}(\mathcal{A})$. The derivation of $T^*$ has been relegated to Appendix~\ref{proof_qm1_ub}, which gives the upper bound as stated above.
\end{proof}

\subsection{Query Complexity Lower bound} \label{qm1_lb}

\begin{theorem}\label{lem_qm1_lb} 
For any $\delta$-true mode estimator $\mathcal{A}$, we have,
\begin{equation*}
\mathbb{E}[{Q}_{\delta}^{\mathcal{P}}(\mathcal{A})] \geq \frac{p_1}{(p_1-p_2)^2}\log (1/2.4\delta).
\end{equation*}
\end{theorem}
\begin{proof}
Let $x(\mathcal{P}) = \argmax_{i \in [k]} p_i$ denote the mode of the distribution $\mathcal{P}$. By assumption, we have $x(\mathcal{P}) = 1$. Consider any $\mathcal{P}'$ such that $x(\mathcal{P}')\neq x(\mathcal{P})=1$. Let $\mathcal{A}$ be a $\delta$-true mode estimator and let $\hat{m}$ be its output. Then, by definition, 
\begin{gather} 
\begin{aligned}\label{delta_p}
&\mathcal{P}(\hat{m}=1) \geq 1-\delta\\
&\mathcal{P'}(\hat{m}=1) \leq \delta.
\end{aligned}
\end{gather}
Let $\tau$ be the stopping time associated with estimator $\mathcal{A}$. Using Wald's lemma \cite{wald1944cumulative} we have,
\begin{align}
\nonumber 
\mathbb{E}_\mathcal{P} \left[ \sum_{t=1}^{\tau} \log \frac{p(X_t)}{p'(X_t)} \right] &= \mathbb{E}_\mathcal{P}[\tau]. \mathbb{E}_\mathcal{P} \left[ \log \frac{p(X_1)}{p'(X_1)} \right] \\
&=\mathbb{E}_\mathcal{P}[\tau]. D(\mathcal{P}||\mathcal{P}').
\label{eqn:chngmeas}
\end{align}
where $D(\mathcal{P}||\mathcal{P}')$ refers to the Kullback-Leibler (KL) divergence  between the distributions $\mathcal{P}$ and $\mathcal{P'}$.\\
Also,
\begin{align*} 
&\mathbb{E}_\mathcal{P} \left[ \sum_{t=1}^{\tau} \log \frac{p(X_t)}{p'(X_t)} \right] \\&= \mathbb{E}_\mathcal{P} \left[ \log \frac{p(X_1,X_2....X_{\tau})}{p'(X_1,X_2....X_{\tau})} \right]\\
&=D(p(X_1,....X_{\tau})||p'(X_1,....X_{\tau}))\\
&\geq D(Ber(\mathcal{P}(\hat{m} = 1))||Ber(\mathcal{P}'(\hat{m} = 1))),
\end{align*}
where $Ber(x)$ denotes the Bernoulli distribution with parameter $x \in (0,1)$ and the last step is obtained by the data processing inequality \cite{cover2012elements}. Furthermore we have,
$$ 
D(Ber(\mathcal{P}(\hat{m}=1))||Ber(\mathcal{P}'(\hat{m}=1)))\\ \geq \log (1/2.4\delta),
$$
where the above follows from \eqref{delta_p} and \cite[Remark 2]{kaufmann2016complexity}. Finally, combining with \eqref{eqn:chngmeas} and noting that the argument above works for any $\mathcal{P}'$ such that $x(\mathcal{P}) \neq x(\mathcal{P}')$, we have 
\begin{align} \label{inf_b} 
\mathbb{E}_\mathcal{P}[\tau] \geq \frac{\log (1/2.4\delta)}{\inf_{\mathcal{P}':x(\mathcal{P}')\neq x(\mathcal{P})} D(\mathcal{P}||\mathcal{P}')}.
\end{align}
We choose $\mathcal{P}'$ as follows, for some small $\epsilon>0$:
\begin{align*}
p_i'&=p_i, \,\, \forall i>2 \\
p_1'&=\frac{p_1+p_2}{2} -\epsilon = q-\epsilon, \, \, \text{where } q = \frac{p_1+p_2}{2}\\
p_2'&=\frac{p_1+p_2}{2} +\epsilon = q+\epsilon
\end{align*}
We have  
\begin{align*}
D(\mathcal{P}||\mathcal{P}') &= p_1\log \left( \frac{p_1}{q-\epsilon} \right) + p_2\log \left( \frac{p_2}{q+\epsilon} \right) \\
&= (p_1+p_2) \Bigg[ \frac{p_1}{p_1+p_2} \log \left( \frac{\frac{p_1}{p_1+p_2}}{\frac{q-\epsilon}{p_1+p_2}} \right) +\\
& \>\>\>\>\>\>\>\>\>\>\>\>\>\>\>\>\>\>\>\>\>\>\>\>\>\>\>\>\>\>\>\>\>\> \frac{p_2}{p_1+p_2} \log \left( \frac{\frac{p_2}{p_1+p_2}}{\frac{q+\epsilon}{p_1+p_2}} \right) \Bigg] \\
&= (p_1+p_2) D\left( \frac{p_1}{p_1+p_2} || \frac{q-\epsilon}{p_1+p_2} \right)\\ 
&\overset{(a)} \leq (p_1+p_2) \left[ \frac{(p_1-q+\epsilon)^2}{(q-\epsilon)\cdot (q+\epsilon)} \right]\\
&= (p_1+p_2) \left[ \frac{(p_1-p_2+2\epsilon)^2}{(p_1+p_2+2\epsilon)\cdot(p_1+p_2-2\epsilon)} \right]\\
&\overset{(b)}\approx \frac{(p_1-p_2)^2}{(p_1+p_2)} \\ &\leq \frac{(p_1-p_2)^2}{p_1},\stepcounter{equation}\tag{\theequation}\label{dpp_lb}
\end{align*}
where $(a)$ follows from \cite[Theorem 1.4]{popescu2016bounds} which gives an upper bound on the KL divergence and $(b)$ follows because $\epsilon$ can be arbitrarily small. 
Note that an upper bound on the stopping time $\tau$ would also give an upper bound on ${Q}_{\delta}^{\mathcal{P}}(\mathcal{A})$, since we have one query at each time. Hence, using \eqref{inf_b} and \eqref{dpp_lb}, we get the following lower bound:
\begin{equation*} \label{lower} 
\mathbb{E}[{Q}_{\delta}^{\mathcal{P}}(\mathcal{A})]  \geq \frac{p_1}{(p_1-p_2)^2} \log (1/2.4\delta). 
\end{equation*}
\end{proof}

\section{Mode estimation with QM2}\label{sec_QM2}
In this section, we present an algorithm for mode estimation under QM2 and analyse its query complexity. 
\subsection{Algorithm} \label{algo_qm2}
Recall that under the QM2 query model, a query to the oracle with indices $i$ and $j$ reveals whether the samples $X_i$ and $X_j$ have the same value or not. Our mode estimation scheme is presented in Algorithm~\ref{alg2}. During the course of the algorithm we form bins, and the $i^{th}$ bin formed is referred to as bin $i$. Here, each bin is a collection of indices having the same value as revealed by the oracle. Let $\sigma(i)$ denote the element from the support that that $i^{th}$ bin represents. The algorithm proceeds in rounds. In round $t$, let $b(t)$ denote the number of bins present. Based on the observations made so far, we form a subset of bins $\mathcal{C}(t)$. We go over the bins in $\mathcal{C}(t)$ one by one, and in each iteration within the round we query the oracle with index $t$ and a sample index $j_l$ from some bin $l$, i.e.,  $j_l \in \text{ bin } l$ for $l \in \mathcal{C}(t)$. The round ends when the oracle gives a positive answer to one of the queries or we exhaust all bins in $\mathcal{C}(t)$. So, the number of queries in round $t$ is at most $|\mathcal{C}(t)| \leq b(t)$. If we get a positive answer from the oracle, we add $t$ to the corresponding bin. A new bin will be created if the oracle gives a negative answer to each of these queries. \\ 
We will now describe how we construct the subset $\mathcal{C}(t)$ of bins we compare against in round $t$. To do so, for each bin $i$, corresponding to $\sigma(i)$, and for each round $t$, we maintain an empirical estimate of each $p_{\sigma(i)}$ during each round, denoted by $\hat{p}_{\sigma(i)}^t$. We also maintain confidence intervals for each $p_{\sigma(i)}$, denoted by $\beta_{\sigma(i)}^t$. The choice for $\beta_{\sigma(i)}^t$ is the same as that in QM1, given in \eqref{beta}. $\mathcal{C}(t)$ is formed in each round $t$ by considering only those bins whose UCB is greater than the LCB of all other bins. Mathematically, 
\begin{align*} 
\mathcal{C}(t) = \{i \in \{1, 2,& \ldots, b(t)\} \text{ such that } \nexists l \text{ for which }\\ & \hat{p}_l^t - \beta_l^t > \hat{p}_{\sigma(i)}^{t} + \beta_{\sigma(i)}^{t} \} \stepcounter{equation}\tag{\theequation}\label{c_t}
\end{align*}
The rest of the algorithm is similar to Algorithm~\ref{alg1}, including the stopping rule, i.e. we stop when there is a bin whose LCB is greater than the UCB of all other bins in $\mathcal{C}(t)$. As before, the choice of confidence intervals and the stopping rule ensures that when $A_t=\textit{(stop, i)}$, the corresponding element $\sigma(i)=1$ with probability at least $(1-\delta)$.


\begin{algorithm} 
\caption{Mode estimation algorithm under QM2} 
\label{alg2} 
\begin{algorithmic}[1] 
    \STATE $t=1$
    \STATE $A_0=$ \textit{(continue,1)} : Obtain $X_1$.
    \LOOP
        \STATE Form $\mathcal{C}(t)$ according to \eqref{c_t}.
        \STATE flag=0. 
        \FORALL{bins $l \in \mathcal{C}(t)$}
            \STATE Obtain $j_l \in$ bin $i$ .
            \IF{$\mathcal{O}(t,j_l)==+1$}
                \STATE Add $t$ to the corresponding bin.
                \STATE flag=1; \textit{BREAK}
            \ENDIF
        \ENDFOR
        \IF{flag==0}
            \STATE Create a new bin for index $t$.
        \ENDIF
        \STATE Update the empirical estimate $\hat{p}_i^t$ \eqref{p_hat} and confidence interval $\beta_i^t$ \eqref{beta} for all bins $i \in \{1,2,\ldots,k\}$.
        \STATE \begin{equation*}
                 A_t =
                  \begin{cases}
                    \text{\textit{(stop,i)}:} &  \text{if }\exists i, \text{ s.t. } \forall j \in \mathcal{C}(t)\\
                    \vspace{0.15in}\text{Exit, Output $i$}&\hat{p}_i^t - \beta_i^t > \hat{p}_j^t + \beta_j(t), \\
                    \text{\textit{(continue,t+1)}:} & \hspace{0.2in} \text{otherwise } \\
                    \text{Move to next round}
                  \end{cases} 
                \end{equation*}
        \STATE $t=t+1$
    \ENDLOOP
\end{algorithmic}
\end{algorithm}

\subsection{Analysis} \label{analysis_qm2}
\begin{theorem}
For the choice of $\beta_i^t$ as given by \eqref{beta}, Algorithm~\ref{alg2} is a $\delta$-true mode estimator. 
\end{theorem}
\begin{proof}
The error analysis for Algorithm~\ref{alg2} is similar as that for Algorithm~\ref{alg1} as given in Section~\ref{analysis_qm1}. We consider the same two events $\mathcal{E}_1$ and $\mathcal{E}_2$. Recall that $\mathcal{E}_1$ denotes the error event that for some pair $(i,t)$ the confidence bound \eqref{Eqn:confinterv} around $\hat{p}_i^t$ is violated; $\mathcal{E}_2$ denotes the error event that for some pair $(i,t)$ the confidence bound \eqref{Eqn:confinterv_var} around the sample variance $V_t(Z^i)$ is violated. Again choosing similar values for $\delta_1$ and $\delta_2$, we get $\mathbb{P}[\mathcal{E}_1^c \cap \mathcal{E}_2^c] \geq 1 - \delta$. We now need to show that if the event $\mathcal{E}_1^c \cap \mathcal{E}_2^c$ is true, then Algorithm~\ref{alg2} returns $1$ as the mode. This then implies that Algorithm~\ref{alg2} is a $\delta$-true mode estimator. \\
The analysis remains same as discussed for QM1. The only additional factor we need to consider here, is the event that the bin corresponding to the mode $1$ of the underlying distribution is discarded in one of the rounds of the algorithm, i.e., it is not a part of $\mathcal{C}(t)$ for some round $t$. A bin is discarded when its UCB becomes less than the LCB of any other bin. Under event $\mathcal{E}_1^c$, all confidence intervals are true, and since $p_1>p_i, \forall i$ the UCB of the corresponding bin can never be less than the LCB of any other bin. This implies that under $\mathcal{E}_1^c$, the bin corresponding to the mode $1$ is never discarded. \\
Hence, the probability of returning $1$ as the mode is at least $\mathbb{P}[\mathcal{E}_1^c \cap \mathcal{E}_2^c]$, which by \eqref{union} implies that Algorithm~\ref{alg2} is a $\delta$-true mode estimator.
\end{proof}

\subsection{Query Complexity Upper bound}
From the analysis of the sample complexity for the QM1 model derived in Section~\ref{qm1_ub}, we get one upper bound for the QM2 case. Algorithm~\ref{alg1} continues for at most $T^*$ rounds with probability at least $1-\delta$, where $T^*$ is given by Theorem~\ref{lem_qm1_ub}. A sample accessed in each of these rounds can be compared to at most $\min \{ k,T^* \}$ other samples. Thus, a natural upper bound on the query complexity of Algorithm~\ref{alg2} is $(T^*\cdot \min \{ k,T^* \})$. The following result provides a tighter upper bound.
\begin{theorem} \label{lem_qm2_ub}
For a $\delta$-true mode estimator $\mathcal{A}_2$, corresponding to Algorithm~\ref{alg2}, we have the following with probability at least $(1-\delta)$.
\begin{multline*}
 {Q}_{\delta}^{\mathcal{P}}(\mathcal{A}_2) \leq \frac{592}{3} \frac{p_1}{(p_1-p_2)^2} \log \left( \frac{592}{3} \sqrt{\frac{k}{\delta}} \frac{p_1}{(p_1-p_2)^2} \right) \\ +  \sum_{i=2}^{T} \frac{592}{3} \frac{p_1}{(p_1-p_i)^2} \log \left( \frac{592}{3} \sqrt{\frac{k}{\delta}} \frac{p_1}{(p_1-p_i)^2} \right)
\end{multline*}
%
 for $T = \min \left \{ k, \frac{592}{3} \frac{p_1}{(p_1-p_2)^2} \log \left( \frac{592}{3} \sqrt{\frac{k}{\delta}} \frac{p_1}{(p_1-p_2)^2} \right) \right \}$.
\end{theorem}
\begin{proof}
The detailed calculations are provided in Appendix~\ref{proof_qm1_ub}, here we give a sketch. Under the event $\mathcal{E}_1^c \cap \mathcal{E}_2^c$, where all confidence bounds hold true, for any bin represented during the run of the algorithm and corresponding to some element $i \neq 1$ from the support, we have that it will definitely be excluded from $\mathcal{C}(t)$ when its confidence bound $ \beta_{i}^t < \frac{p_1-p_{i}}{4}$ and $ \beta_1^t < \frac{p_1-p_{i}}{4}$. We find a $t$ which satisfies the stopping condition for each of the bins represented, and summing over them gives the total query complexity. 
Following the calculations in Appendix~\ref{proof_qm1_ub}, we get the following value of $t_i^*$ for the bin corresponding to element $i\neq 1$, such that by $t=t_i^*$ the bin will definitely be excluded from $\mathcal{C}(t)$.
\begin{equation*}
t_i^* = \frac{592}{3} \frac{p_1}{(p_1-p_i)^2} \log \left( \frac{592}{3} \sqrt{\frac{2k}{\delta}} \frac{p_1}{(p_1-p_i)^2} \right)
\end{equation*} 
Also, for the first bin we get $t_1^*$ as follows.
\begin{equation*}
t_1^* = \frac{592}{3} \frac{p_1}{(p_1-p_2)^2} \log \left( \frac{592}{3} \sqrt{\frac{2k}{\delta}} \frac{p_1}{(p_1-p_2)^2} \right)
\end{equation*}
Also, the total number of bins created will be at most \\ $ T = \min \left \{ k, \frac{592}{3} \frac{p_1}{(p_1-p_2)^2} \log \left( \frac{592}{3} \sqrt{\frac{2k}{\delta}} \frac{p_1}{(p_1-p_2)^2} \right) \right \}$. \\ \\ A sample from the bin corresponding to element $i$ will be involved in at most $t_i^*$ queries. Hence the total number of queries, ${Q}_{\delta}^{\mathcal{P}}(\mathcal{A})$, is bounded as follows
\begin{multline*}
    {Q}_{\delta}^{\mathcal{P}}(\mathcal{A}) \leq \frac{592}{3} \frac{p_1}{(p_1-p_2)^2} \log \left( \frac{592}{3} \sqrt{\frac{k}{\delta}} \frac{p_1}{(p_1-p_2)^2} \right) + \\ \sum_{i=2}^{T} \frac{592}{3} \frac{p_1}{(p_1-p_i)^2} \log \left( \frac{592}{3} \sqrt{\frac{k}{\delta}} \frac{p_1}{(p_1-p_i)^2} \right)
\end{multline*}
\end{proof}

\subsection{Query Complexity Lower bound}
The following theorem gives a lower bound on the expected query complexity for the QM2 model.
\begin{theorem}\label{lem_qm2_lb}
For any $\delta$-true mode estimator $\mathcal{A}$, we have,
\begin{equation*}
\mathbb{E}[{Q}_{\delta}^{\mathcal{P}}(\mathcal{A})] \geq \frac{p_1}{2(p_1-p_2)^2}\log (1/2.4\delta).
\end{equation*}
\end{theorem}
\begin{proof}
Consider any $\delta$-true mode estimator $\mathcal{A}$ under query model QM2 and let $\tau$ denote its average (pairwise) query complexity when the underlying distribution is $\mathcal{P}$. Next, consider the QM1 query model and let estimator $\mathcal{A}^{'}$ simply simulate the estimator $\mathcal{A}$ under QM2, by querying the values of any sample indices involved in the pairwise queries. It is easy to see that since $\mathcal{A}$ is a $\delta$-true mode estimator under QM2, the same will be true for $\mathcal{A}^{'}$ as well under QM1. Furthermore, the expected query complexity of $\mathcal{A}^{'}$ under QM1 will be at most $2\tau$ since each pairwise query involves two sample indices. Thus, if the query complexity $\tau$ of $\mathcal{A}$ under QM2 is less than $\frac{p_1}{2(p_1-p_2)^2}\log (1/2.4\delta)$, then we have a $\delta$-true mode estimator under query model QM1 with query complexity less than $\frac{p_1}{(p_1-p_2)^2}\log (1/2.4\delta)$, which contradicts the lower bound in Theorem~\ref{lem_qm1_lb}. The result then follows. 
\end{proof}
\noindent The lower bound on the query complexity as given by the above theorem matches the first term of the upper bound given in Theorem~\ref{lem_qm2_ub}. So, the lower bound will be close to the upper bound when the first term dominates, in particular when $\frac{p_1}{(p_1-p_2)^2} \gg \sum_{i=3}^k \frac{p_1}{(p_1-p_i)^2}$. \\

While the above lower bound matches the upper bound in a certain restricted regime, we would like a sharper lower bound which works more generally. Towards this goal, we consider a slightly altered (and relaxed) problem setting which relates to the problem of best arm identification in a multi-armed bandit (MAB) setting studied in \cite{kaufmann2016complexity}, \cite{soare2014best}. The altered setting and the corresponding result are discussed in Appendix~\ref{proof_multi_armed_lb}.

\section{Experimental Results}\label{sec:exp}
For both the QM1 and QM2 models, we simulate Algorithm~\ref{alg1} and Algorithm~\ref{alg2} for various synthetic distributions. We take $k=5120$ and keep the difference  $p_1 - p_2 = 0.208$ constant for each distribution. For the other $p_i$'s we follow two different models:
\begin{enumerate}
\item Uniform distribution : The other $p_i$'s for $i=3....k$ are chosen such that each $p_i = \frac{1-p_1-p_2}{k-2}$.
\item Geometric distribution : The other $p_i$'s are chosen such that $p_2,p_3....p_k$ form a decreasing geometric distribution which sums upto $1-p_1$.
\end{enumerate}
For each distribution we run the experiment 50 times and take an average to plot the query complexity. In Fig.~\ref{fig:fig1} we plot the number of queries taken by Algorithm~\ref{alg1} for QM1, for both the geometric and uniform distributions. As suggested by our theoretical results, the query complexity increases (almost) linearly with $p_1$ for a fixed $(p_1 - p_2)$. In Fig.~\ref{fig:fig2} we plot the number of queries taken by Algorithm~\ref{alg2} for QM2 and compare it to the number of queries taken by a naive algorithm which queries a sample with all the bins formed, for both the uniform and geometric distributions.  \\
To further see how Algorithm~\ref{alg2} performs better than the naive algorithm which queries a sample against all bins formed, we plot the number of queries in each round for a particular geometric distribution, in Fig.~\ref{fig:fig3}. We observe that over the rounds, for Algorithm~\ref{alg2} bins start getting discarded and hence the number of queries per round decreases, while for the naive algorithm, as more and more bins are formed, the number of queries per round keeps increasing.\\
\begin{figure} 
\centering\includegraphics[width=.95\columnwidth]{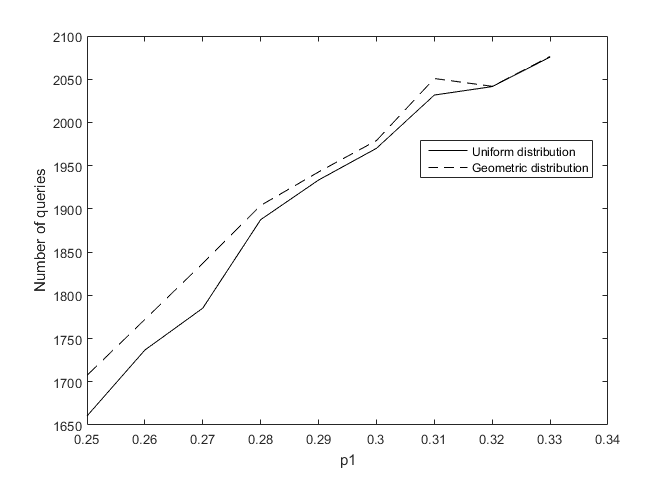}
\caption{Number of queries for Algorithm~\ref{alg1} under the uniform and geometric distributions.}\label{fig:fig1}
\end{figure}

\begin{figure} 
\centering\includegraphics[width=.95\columnwidth]{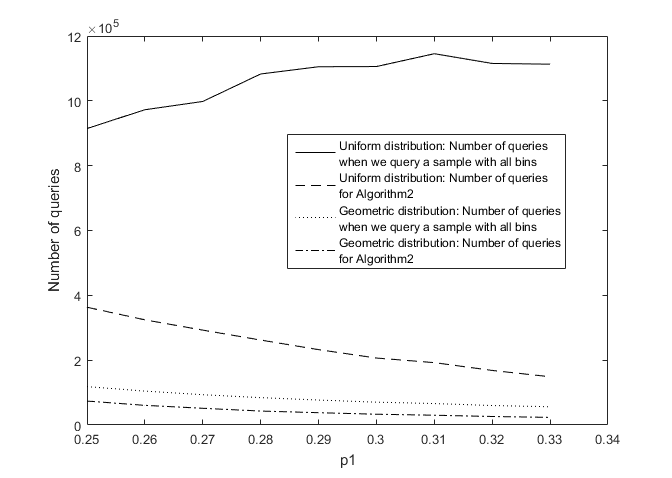}
\caption{Number of queries when each sample is queried with all the bins, and number of queries for Algorithm~\ref{alg2} for the uniform and geometric distributions.}\label{fig:fig2}
\end{figure}

\begin{figure} 
\centering\includegraphics[width=.95\columnwidth]{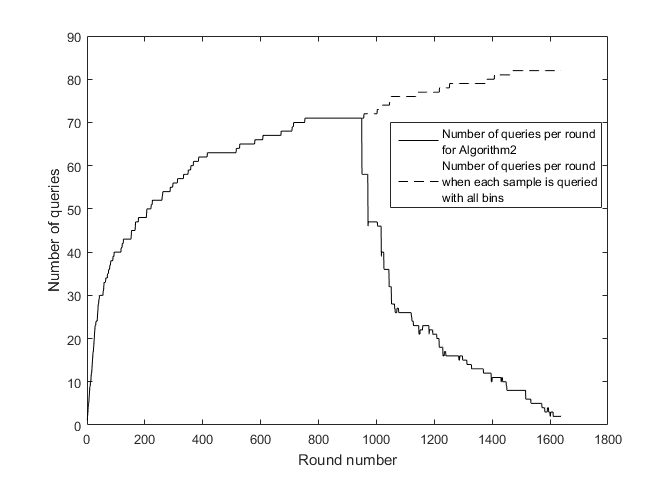}
\caption{Number of queries per round for Algorithm~\ref{alg2} and for the algorithm which queries each sample with all of the bins formed.}\label{fig:fig3}
\end{figure}

\textbf{Real world dataset:} As mentioned in the introduction, one of the applications of mode estimation is partial clustering. Via experiments on a real-world purchase data set \cite{snapnets}, we were able to benchmark the performance of our proposed Algorithm~\ref{alg2} for pairwise queries, a naive variant of it with no UCB-based bin elimination and the algorithm of \cite{mazumdar2017clustering} which performs a full clustering of the nodes. Using \cite{snapnets}, we create a dataset where each entry represents an item available on Amazon.com along with a label corresponding to a category that the item belongs to. We take entries with a common label to represent a cluster and we consider the task of using pairwise oracle queries to identify a cluster of items within the top-10 clusters, from among the top-100 clusters in the entire dataset. Note that while our algorithm is tailored towards finding a heavy cluster, the  algorithm in \cite{mazumdar2017clustering} proceeds by first learning the entire clustering and then identifying a large cluster. Some statistics of the dataset over which we ran our algorithm are: number of clusters - 100; number of nodes - 291925; size of largest cluster - 53551; and size of 11th largest cluster - 19390.

With a target confidence of 99\%, our proposed algorithm terminates in $\sim 631k$ pairwise queries while the naive algorithm takes $\sim 1474k$ queries ($2.3x$ more). The algorithm of \cite{mazumdar2017clustering} clusters all nodes first, and is thus expected to take around $n*k = 29192.5k$ queries ($46x$ more) to terminate; this was larger than our computational budget and hence could not be run successfully. With a target confidence of 95\%, our algorithm takes $\sim 558k$ queries instead of the naive version which takes $\sim 1160k$ queries ($2x$ more) to terminate.

\section{Discussion} \label{discussion}
There are a few variations of the standard mode estimation problem which have important practical applications and we discuss them in the following subsections. 

\subsection{Top-$m$ estimation}
An extension of the mode estimation problem discussed could be estimating the top-$m$ values. i.e. for $p_1>p_2>....>p_k$, an algorithm should return the set $\{1,2,...m \}$. A possible application is in clustering, where we are interested in finding the largest $m$ clusters. The algorithms \ref{alg1} and \ref{alg2} for QM1 and QM2 would change only in the stopping rule. The new stopping rule would be such that $A_t =$ \textit{(stop,.)} when there exist $m$ bins such that their LCB is greater than the UCB of the other remaining bins. In this setting, we define a $\delta$-true top-$m$ estimator as an algorithm which returns the set $\{1,2,....m \}$ with probability at least $(1-\delta)$. In the following results we give bounds on the query complexity, ${Q}_{\delta}^{\mathcal{P}}(\mathcal{A})$ for a $\delta$-true top-$m$ estimator $\mathcal{A}$, for the QM1 query model.

\begin{theorem} 
For a $\delta$-true top-$m$ estimator $\mathcal{A}_m$, corresponding to Algorithm~\ref{alg1} for the top-$m$ case, we have the following with probability at least $(1-\delta)$.
\begin{align*}
&Q_{\delta}^{\mathcal{P}}(\mathcal{A}_m) \leq \\ & \max_{\substack{i\in \{1,2....m \} \\ j\in \{m+1,....k \}}} \frac{592}{3} \frac{p_i}{(p_i-p_j)^2} \log \left( \frac{592}{3} \sqrt{\frac{k}{\delta}} \frac{p_i}{(p_i-p_j)^2} \right)
\end{align*}
\end{theorem} 
\begin{proof}
The proof follows along the same lines as Theorem~\ref{lem_qm1_ub}. Here, for a bin $i\in \{ 1,2,...m\}$ and a bin $j\in \{m+1,....k \}$, their confidence bounds would be separated when $\beta_i^t < \frac{p_i-p_j}{4}$ and $\beta_j^t < \frac{p_i-p_j}{4}$. The calculations then follow similarly as before to give the above upper bound.
\end{proof}

\begin{theorem}
For any $\delta$-true top-$m$ estimator $\mathcal{A}$, we have,
\begin{equation*}\label{top_m_lb}
\mathbb{E}[{Q}_{\delta}^{\mathcal{P}}(\mathcal{A})] \geq \max_{\substack{i\in \{1,2....m \} \\ j\in \{m+1,....k \}}} \frac{p_i}{(p_i-p_j)^2} \log \left(1/2.4\delta\right)
\end{equation*}
\end{theorem}
\begin{proof}
The proof follows along the same lines as Theorem~\ref{lem_qm1_lb}. Here, for $\epsilon>0$, the alternate distribution $\mathcal{P}'$ that we choose would have $p_i'=\frac{p_i+p_j}{2} - \epsilon$  and $p_j'=\frac{p_i+p_j}{2} + \epsilon$ for some $i\in \{1,2....m \}$ and some $j\in \{m+1,....k \}$. We take the maximum value over all such $i,j$ to give the lower bound. 
\end{proof}

\subsection{Noisy oracle}
Here, we consider a setting where the oracle answers queries noisily and we analyze the impact of errors on the query complexity for mode estimation. \\
For the noisy QM1 model, we assume that when we query the oracle with some index $j$, the value revealed is the true sample value $X_j$ with probability $(1- p_e)$ and any of the $k$ values in the support with probability $\frac{p_e}{k}$ each. The problem of mode estimation in this noisy setting is equivalent to that in a noiseless setting where the underlying distribution is given by 
$$
p_i'= (1-p_e)p_i + \frac{p_e}{k}.
$$
Since the mode of this altered distribution is the same as the true distribution, we can use Algorithm~\ref{alg1} for mode estimation under the noisy QM1 model, and the corresponding query complexity bounds in Section~\ref{sec_QM1} hold true.\\
For the noisy QM2 model, we assume that for any pair of indices $i$ and $j$ sent to the oracle, it returns the correct answer ($+1$ if $X_i = X_j$, $- 1$ otherwise) with probability $(1 - p_e)$. This setting is technically more involved and is in fact similar to clustering using pairwise noisy queries \cite{mazumdar2017clustering}. The mode estimation problem in this setting corresponds to identifying the largest cluster, which has been studied in \cite{Shah1902:Top}. Deriving tight bounds for this case is still open. 

\bibliography{bibfile}
\bibliographystyle{aaai}

\appendix

\section{Empirical Bernstein result} \label{emp_bern}
We use the following result given in \cite{maurer2009empirical} to get the confidence bounds. Since we do not know the underlying distribution and it's variance, we cannot use the standard Bernstein bound. The following result gives us a handle on the confidence bounds, using empirical variance. 

\begin{theorem} \label{bern_thm}
Let $X_1,X_2,...X_t$ be i.i.d. random variables with values in $[0,1]$ and let $\delta>0$. Then with probability at least $1-\delta$ in the i.i.d. vector $\mathbf{X}=(X_1,X_2....X_t)$ we have
\begin{equation*}
\Bigl\lvert \mathbb{E}X - \frac{1}{t}\sum_{i=1}^t X_i \Bigr\rvert  \leq \sqrt{\frac{2V_t(\mathbf{X})\log (2/\delta)}{t}} + \frac{7 \log (2/\delta)}{3(t-1)},
\end{equation*}
where $V_t(\mathbf{X})$ is the sample variance 
\begin{equation*}
V_t(\mathbf{X}) = \frac{1}{t(t-1)}\sum_{1\leq i<j\leq t} (X_i-X_j)^2.
\end{equation*}
\end{theorem}
Also, the following result also given in \cite{maurer2009empirical} is used to establish bounds on the sample variance.
\begin{theorem} \label{var}
Let $t\geq 2$ and $\mathbf{X}=(X_1,X_2....X_t)$ be a vector of independent random variables with values in $[0,1]$. Then for $\delta>0$ and $\mathbb{E}V_t = \mathbb{E}_\mathbf{X} V_t(\mathbf{X})$ we have, 
\begin{align*}
\text{Pr}\left\{ \sqrt{\mathbb{E}V_t} > \sqrt{V_t(\mathbf{X})} + \sqrt{\frac{2 \log 1/\delta_1}{t-1}} \right \} \leq \delta \\
\text{Pr}\left\{ \sqrt{V_t(\mathbf{X})} > \sqrt{\mathbb{E}V_t} + \sqrt{\frac{2 \log 1/\delta_1}{t-1}} \right \} \leq \delta
\end{align*}
\end{theorem}

\section{Details in proof of Theorem 2} \label{proof_qm1_ub}
Under $\mathcal{E}_1^c \cap \mathcal{E}_2^c$ the confidence intervals around $\hat{p}_i^t$ and $V_t(Z^i)$ hold true for all pairs $(i,t)$, and we derive the value of $T^*$ which would ensure that the confidence intervals of the bins are separated enough to ensure that $1$ is returned as the mode.\\
For each bin $i \neq 1$ created during the run of the algorithm, its UCB will definitely be less than the LCB of the first bin when $\beta_i^t < \frac{p_1-p_i}{4}$ and $\beta_1^t < \frac{p_1-p_i}{4}$. This happens because 
\begin{align}\label{stopp}
&\beta_i^t < \frac{p_1-p_i}{4} \text{ and } \beta_1^t < \frac{p_1-p_i}{4} \nonumber\\
&\implies \beta_i^t + \beta_1^t < \frac{p_1-p_i}{2}\\
&\implies \hat{p}_i^t + \beta_i^t < \hat{p}_1^t - \beta_1^t \nonumber
\end{align} 
We find a time $t_i$ which satisfies the stopping condition for each bin $i$ created, and taking the maximum over these gives $T^*$.
For the $i^{th}$ bin we have, 
\begin{align*}
&\beta_i^t < \frac{p_1-p_i}{4} \\
\implies &\sqrt{\frac{2V_t(Z^i)\log (4kt^2/\delta)}{t}} + \frac{7 \log (4kt^2/\delta)}{3(t-1)} < \frac{p_1-p_i}{4}\\
& \text{Using the bound around the empirical variance and }\\
& \text{again choosing $\delta_1=\frac{\delta}{2kt^2}$ we get,}\\
\implies &\sqrt{\frac{2\log (4kt^2/\delta)}{t}} \left [ \sqrt{p_i(1-p_i)} + \sqrt{\frac{2\log (4kt^2/\delta)}{t-1}} \right ] \\
&\>\>\>\>\>\>\>\>\>\>\>\>\>\>\>\>\>\>\>\>\>\>\>\> + \frac{7 \log (4kt^2/\delta)}{3(t-1)} < \frac{p_1-p_i}{4} \\
\implies & \frac{13\log (4kt^2/\delta)}{3t} + \sqrt{p_i(1-p_i)} \sqrt{\frac{2\log (4kt^2/\delta)}{t}} \\
&\>\>\>\>\>\>\>\>\>\>\>\>\>\>\>\>\>\>\>\>\>\>\>\>\>\>\>\>\>\>\>\>\>\>\>\>\>\>\> < \frac{p_1-p_i}{4}\\
&\text{Let } \alpha = \frac{t}{2\log (4kt^2/\delta)} \\
\implies & \frac{p_1-p_i}{4} \alpha - \sqrt{p_i(1-p_i)} \sqrt{\alpha} -\frac{13}{6} >0 \\
\end{align*}
Solving the above quadratic, we get that for the following values of $\alpha$, the above inequality holds.
\begin{align*}
\sqrt{\alpha} &> \frac{\sqrt{p_i(1-p_i)} + \sqrt{p_i(1-p_i)+\frac{13}{6}(p_1-p_i)} }{(p_1-p_i)/2}.
\end{align*}
We need to choose a value of $\alpha$, such that the above holds. We choose $\alpha$ as follows:
\begin{align*}
\alpha 
&=\frac{74}{3} \frac{p_1}{(p_1-p_i)^2}.
\end{align*}
Hence, we have
\begin{equation*}
\frac{t}{\log \left(2\sqrt{\frac{2k}{\delta}}t \right)} > \frac{296}{3} \frac{p_1}{(p_1-p_i)^2}.
\end{equation*}
The following value for $t^*$, is sufficient such that the above inequality is satisfied when we have at most $t^*$ samples.
\begin{equation*}
t^* = \frac{592}{3} \frac{p_1}{(p_1-p_i)^2} \log \left( \frac{592}{3} \sqrt{\frac{k}{\delta}} \frac{p_1}{(p_1-p_i)^2} \right)
\end{equation*} 
The above value of $t^*$ for bin $i$, is such that for some $t\leq t^*$ the confidence intervals of the $i^{th}$ and first bin are sufficiently separated. Similar calculations can be done for ensuring $\beta_1^t < \frac{p_1-p_i}{2}$. Taking the maximum value, we get $T^*$ as follows, which ensures that under the event $\mathcal{E}_1^c \cap \mathcal{E}_2^c$, the UCB of all bins are less than the LCB of the first bin.
\begin{equation*}
T^*= \frac{592}{3} \frac{p_1}{(p_1-p_2)^2} \log \left( \frac{592}{3} \sqrt{\frac{k}{\delta}} \frac{p_1}{(p_1-p_2)^2} \right)
\end{equation*}

%
%
%
%
%
%
%
%

\section{QM2 Lower Bound altered setting} \label{proof_multi_armed_lb}

The slightly altered setting we study is as follows. We have $k$ bins and a representative element from each of the $k$ bins. Any algorithm proceeds in rounds. In round $t$, the index $t$ is compared with a representative index from each of the bins belonging to $\mathcal{C}(t)$, where $\mathcal{C}(t)$ is any chosen subset of the $k$ bins. Thus, number of queries in round $t$ is the size of $\mathcal{C}(t)$. The oracle response in any round could be a $+1$ for bin $j \in \mathcal{C}(t)$ and a $-1$ for the other bins with probability $p_j$ or it could be a $-1$ for all bins in $\mathcal{C}(t)$ with probability $1 - \sum_{j \in \mathcal{C}(t)} p_j$. Based on the oracle responses, the algorithm either decides to proceed to the next round or stops and outputs an estimate for the mode of the underlying distribution. \\
This setting is slightly different from our original problem setup in the QM2 query model. Essentially, here the support size $k$ is known and we have access to one sample from each bin apriori. Secondly, in a round, in the original setting, we would stop as soon as the oracle gave a positive response. However, here the set of representative samples from a subset of bins $\mathcal{C}(t)$ to be compared against is decided at the beginning of a round $t$ and all the $|\mathcal{C}(t)|$ queries are made in parallel.\\
The above setup can be alternatively viewed as a \textit{structured} MAB setting where there are $k$ arms and each arm $i$ is associated with a reward distribution which is a Bernoulli distribution with mean $p_i$. Note that the means of arms sum up to $1$ since $\sum_{i=1}^k p_i = 1$. The goal is to identify the best arm, i.e. the one with the highest mean value, $p_i$. Any algorithm proceeds in rounds and in round $t$, a subset of arms, $\mathcal{C}(t)$, are pulled. The output is a vector which has $+1$ for arm $j \in \mathcal{C}(t)$ and $-1$ for everyone else, with probability $p_j$ or the output is a vector with $-1$ for all arms in $\mathcal{C}(t)$ with probability $1 - \sum_{j \in \mathcal{C}(t)} p_j$. Based on these responses, the algorithm either decides to proceed to the next round or stops and outputs an estimate for the best arm. The number of arms that need to be pulled to guarantee that the estimate is correct with probability at least $(1-\delta)$ will correspond to the query complexity of a $\delta$-true mode estimator in the above described setting. \\
While we are not able to provide a lower bound for the structured MAB setting above, we present a lower bound in the following relaxed setting. Instead of requiring that the expected rewards for individual arms lie on a simplex, i.e.,
\begin{equation}
\label{eqn:simplex}
p_1+p_2+.....p_k = 1,
\end{equation}
we prove our lower bound for a slightly relaxed setting where the means of the arms satisfy : 
\begin{equation}
\label{condition} 
2p_1+p_2+.....p_k < 1.
\end{equation}
\begin{theorem} \label{multi_armed_lb}
For a MAB setting in which the expected rewards for individual arms satisfy the condition given in \eqref{condition}, any algorithm which correctly identifies the best arm with probability at least $(1-\delta)$, has the following lower bound on the total number of arm pulls, $N$:
\begin{align*}
\mathbb{E}_{\mathcal{P}}[N] \geq \Bigg [ \sum_{i=2}^k \frac{p_i}{2(p_1-p_i)^2} \Bigg ] \log (1/2.4\delta)
\end{align*}
\end{theorem}
\begin{proof}
Our proof follows along similar lines as that for \cite[Lemma1]{kaufmann2016complexity}. Consider any estimator $\mathcal{A}$ which can correctly identify the best arm with probability of error at most $\delta$. Let $\mathcal{P}$ and $\mathcal{P}'$ be two bandit models with $k$ arms and means for the reward distributions as follows:
\begin{align*}
\mathcal{P} = (p_1,p_2,p_3,.....p_k) \\
\mathcal{P}' = (p_1,p_1+\epsilon,p_3,.....p_k)
\end{align*} 
for some $\epsilon>0$. Note that the best arms in the two models are arms $1$ and $2$ respectively.

Recall that we can choose to pull any subset $S$ of arms in each round, where $S$ is one of $2^k$ possibilities. On pulling a subset $S$, the possible output is a vector with +1 for some arm $j \in S$ and $-1$ for everyone else with probability $p_j$ or the output is a $-1$ for all arms in $S$ with probability $1 - \sum_{j\in S} p_j$. Let $(Y_{S_a,s})$ be the sequence of i.i.d. vectors observed on pulling the $a^{th}$ subset $S_a$. Based on the observations made till round $t$, the likelihood ratio $L_t$ can be written as
\begin{align*}
L_t = \sum_{a=1}^{2^k} \sum_{s=1}^{N_{S_a}(t)} \log \left ( \frac{f_{S_a}(Y_{S_a,s})}{f'_{S_a}(Y_{S_a,s})} \right)
\end{align*}
where $N_{S_a}(t)$ is the number of times the subset $S_a$ is pulled upto time $t$. With some abuse of notation, we will let $N_{i}(t)$ denote the total number of times arm $i$ was pulled (as part of any subset) upto time $t$.\\
Also, 
\begin{align*}
\mathbb{E}_{\mathcal{P}} \left [ \log \left ( \frac{f_{S_a}(Y_{S_a,s})}{f'_{S_a}(Y_{S_a,s})} \right) \right ] = D(p_{S_a},p'_{S_a})
\end{align*}
Applying Wald's lemma to $L_{\sigma}$, where $\sigma$ is the stopping time associated with the estimator $\mathcal{A}$, gives
\begin{align}
\nonumber
\mathbb{E}_{\mathcal{P}}[L_{\sigma}] = & \sum_{a=1}^{2^k} \mathbb{E}_{\mathcal{P}}[N_{S_a}(\sigma)]D(p_{S_a},p'_{S_a}) \\
&\overset{(a)}\leq \mathbb{E}_{\mathcal{P}}[N_2(\sigma)] \max_{S_a:2 \in S_a} (D(p_{S_a},p'_{S_a}))
\label{eqn:llrub}
\end{align}
where $(a)$ follows because $D(p_{S_a},p'_{S_a})$ is non-zero for only those subsets which contain the arm $2$.\\
Let $S_a$ be such a subset containing arm $2$, and other arms such that the sum of means of other arms in $S_a$ is $s$.  We have 
\begin{align*}
D(p_{S_a},p'_{S_a}) &= p_2 \log \left( \frac{p_2}{p_1+\epsilon} \right)\\ &+ (1-(p_2+s)) \log \left ( \frac{1-(p_2+s)}{1-(p_1+\epsilon+s)} \right )
\end{align*}
The second term in the above is an increasing term with $s$ and so $D(p_{S_a},p'_{S_a})$ takes the maximum value when $s=p_1+p_3+....p_k$. Hence
\begin{align*}
&\max_{S_a:2 \in S_a} (D(p_{S_a},p'_{S_a}))\\
&= p_2 \log \left( \frac{p_2}{p_1+\epsilon} \right) \\
& \hspace{0.3in}+ (1-\sum p_i) \log \left ( \frac{(1-\sum p_i)}{(1-\sum p_i)+p_2 - p_1 - \epsilon } \right ) \\
& \overset{(a)} \leq p_2 \log \left( \frac{p_2}{p_1+\epsilon} \right) + p_1 \log \left( \frac{p_1}{p_2-\epsilon} \right) \\
& \leq (p_1 + p_2) \Bigg[ \frac{p_2}{p_1 + p_2} \log \left( \frac{\frac{p_2}{p_1 + p_2}}{\frac{p_1+\epsilon}{p_1 + p_2}} \right) \\
& \hspace{1.4in} + \frac{p_1}{p_1 + p_2} \log \left( \frac{\frac{p_1}{p_1 + p_2}}{\frac{p_2-\epsilon}{p_1 + p_2}} \right) \Bigg ] \\
& = (p_1 + p_2) D \left( \frac{p_2}{p_1 + p_2} || \frac{p_1+\epsilon}{p_1 + p_2} \right)\\
& \overset{(b)} \leq (p_1 + p_2) \left [ \frac{(p_1-p_2+\epsilon)^2}{(p_1+\epsilon)\cdot (p_2+\epsilon)} \right ] \\
& \overset{(c)} \approx (p_1 + p_2)\frac{(p_1-p_2)^2}{p_1\cdot p_2} \leq \frac{2(p_1-p_2)^2}{p_2}
\end{align*}
where $(a)$ follows because $x\log\left( \frac{x}{x+\gamma} \right)$ is decreasing with $x$ and we also have $\left( 1-\sum_{i=1}^k p_i \right) > p_1$, $(b)$ follows from \cite[Theorem 1.4]{popescu2016bounds} and $(c)$ follows since we can choose an arbitarily small $\epsilon$. Thus, from \eqref{eqn:llrub}, we have 
$$
\mathbb{E}_{\mathcal{P}}[L_{\sigma}] \le \mathbb{E}_{\mathcal{P}}[N_2(\sigma)] \cdot \frac{2(p_1-p_2)^2}{p_2} .
$$
On the other hand, using \cite[Lemma 19]{kaufmann2016complexity}, we also have that for estimator $\mathcal{A}$ which can correctly identify the top arm with probability of error at most $\delta$,
$$
\mathbb{E}_{\mathcal{P}}[L_{\sigma}] \ge \log (1/2.4\delta).
$$
Combining the two inequalities, we get the following lower bound on the number of times the second arm is pulled: 
\begin{align*}
\mathbb{E}_{\mathcal{P}}[N_2(\sigma)] \geq \frac{p_2}{2(p_1-p_2)^2} \log (1/2.4\delta).
\end{align*}
To get a lower bound on the total number of arm pulls, $N$, we can repeat the above argument for each arm $i \neq 1$, which gives the following lower bound:
\begin{align*}
\mathbb{E}_{\mathcal{P}}[N] \geq \sum_{i=2}^{k} \mathbb{E}_{\mathcal{P}}[N_i(\sigma)] \ge \sum_{i=2}^k \frac{p_i}{2(p_1-p_i)^2} \log (1/2.4\delta).
\end{align*}
\end{proof}

Note that the expression in the lower bound above is very similar to that in the upper bound from Theorem 5. Proving such a lower bound for the structured MAB problem with the true simplex constraint in \eqref{eqn:simplex} remains part of our future work. There has been some encouraging recent work which can shed light on this problem. In particular, say we restrict attention to schemes which pull one arm in each round, i.e. $|\mathcal{C}(t)| = 1, \forall t$. In this case, a lower bound on the total number of arm pulls under the constraint  $p_1+p_2+.....p_k = 1$ follows from \cite{simchowitz2017simulator} and the expression indeed matches that in the upper bound of Theorem 5.

\end{document}